\documentclass[sigconf]{acmart}

\AtBeginDocument{%
  }

\copyrightyear{2026}
\acmYear{2026}
\setcopyright{cc}
\setcctype{by}
\acmConference[CIKM '26]{Proceedings of the 35th ACM International Conference on Information and Knowledge Management}{November 07--11, 2026}{Rome, Italy}
\acmBooktitle{Proceedings of the 35th ACM International Conference on Information and Knowledge Management (CIKM '26), November 07--11, 2026, Rome, Italy}
\acmDOI{10.1145/3799682.3840588}
\acmISBN{979-8-4007-2539-5/2026/11}

\usepackage{makecell} 
\usepackage{caption}

\usepackage{algorithm}      
\usepackage{algorithmic}    

\usepackage{booktabs}
\usepackage{multirow}
\usepackage{colortbl}
\usepackage{enumitem}
\usepackage{xcolor}
\usepackage{pdfpages}
\usepackage{amsmath}
\usepackage{graphicx}
\usepackage{gensymb}
\usepackage{subcaption}

\definecolor{background_gray}{gray}{0.9}

\usepackage{amssymb}

\usepackage{amsthm}             
\newtheorem{theorem}{Theorem}   
\newtheorem{definition}{Definition} 

\usepackage{newfloat}
\usepackage{listings}

\begin{document}

\title{A Game-Theoretic Spatio-Temporal Reinforcement Learning Framework for Collaborative Public Resource Allocation}

\author{Songxin Lei}
\authornote{Both authors contributed equally to this research.}
\affiliation{%
  \institution{The Hong Kong University of Science and Technology (Guangzhou)}
  \city{Guangzhou}
  \country{China}}
\email{slei924@connect.hkust-gz.edu.cn}

\author{Qiongyan Wang}
\authornotemark[1]
\affiliation{%
  \institution{The Hong Kong University of Science and Technology (Guangzhou)}
  \city{Guangzhou}
  \country{China}}
\email{qwang650@connect.hkust-gz.edu.cn}

\author{Yanchen Zhu}
\affiliation{%
  \institution{The Hong Kong University of Science and Technology (Guangzhou)}
  \city{Guangzhou}
  \country{China}}
\email{yzhu156@connect.hkust-gz.edu.cn}

\author{Hanyu Yao}
\affiliation{%
  \institution{Peking University}
  \city{Beijing}
  \country{China}}
\email{hanyuyao25@stu.pku.edu.cn}

\author{Sijie Ruan}
\affiliation{%
  \institution{Beijing Institute of Technology}
  \city{Beijing}
  \country{China}}
\email{sjruan@bit.edu.cn}

\author{Weilin Ruan}
\affiliation{%
  \institution{The Hong Kong University of Science and Technology (Guangzhou)}
  \city{Guangzhou}
  \country{China}}
\email{wruan792@connect.hkust-gz.edu.cn}

\author{Yuyu Luo}
\affiliation{%
  \institution{The Hong Kong University of Science and Technology (Guangzhou)}
  \city{Guangzhou}
  \country{China}}
\email{yuyuluo@hkust-gz.edu.cn}

\author{Huaming Wu}
\affiliation{%
  \institution{Tianjin University}
  \city{Tianjin}
  \country{China}}
\email{whming@tju.edu.cn}

\author{Yuxuan Liang}
\authornote{Corresponding author.}
\affiliation{%
  \institution{The Hong Kong University of Science and Technology (Guangzhou)}
  \city{Guangzhou}
  \country{China}}
\email{yuxliang@outlook.com}

\renewcommand{\shortauthors}{Songxin Lei et al.}

\begin{abstract}
    Public resource allocation involves distributing resources, including urban infrastructure, energy, and transportation, which are typically limited in capacity, to meet social demands. 
    In real-world scenarios, resources are typically limited in capacity, which makes coordination among multiple resources essential. 
    However, existing methods often optimize resource movements in an isolated manner and do not explicitly account for capacity-aware collaboration under spatio-temporal dynamics.
    To address this limitation, 
    we introduce the \underline{C}ollaborative \underline{P}ublic \underline{R}esource \underline{A}llocation (CPRA) problem,
    and propose a \underline{G}ame-Theoretic \underline{S}patio- \underline{T}emporal \underline{R}einforcement \underline{L}earning (GSTRL) framework to solve it. Our contributions are twofold: 1) We formulate CPRA as a potential game and 
    construct the potential function based on the objective function of CPRA,
    laying a theoretical foundation for approximating the Nash equilibrium of this NP-hard problem; and 2) Our GSTRL framework effectively captures the spatio-temporal dynamics of the overall system. We evaluate GSTRL on two real-world datasets, where experiments show its superior performance. Our source codes are available at \url{https://github.com/thunderlrr/GSTRL}.
\end{abstract}




\maketitle

\section{Introduction}

Effective public resource allocation is a cornerstone of equitable and resilient urban development. From emergency response~\cite{Feng2024mov} and disaster relief~\cite{Wang2021intro} to mobile healthcare~\cite{Liu2024mov} and traffic control~\cite{Ji2016mov}, the ability to dynamically deploy limited resources directly impacts the safety, well-being, and quality of life of millions of residents. 
In particular, underserved communities, high-density zones, and vulnerable populations are disproportionately affected by inefficiencies in service delivery~\cite{Yang2024intro}. Because traditional static systems often rely on historical averages and pre-defined fixed locations, they are inherently ill-equipped to handle sudden demand surges or unexpected spatial shifts in crowd flow. Consequently, as cities grow increasingly complex and real-time demands fluctuate rapidly across space and time, these static approaches inevitably fall short in meeting such critical needs.


To ensure responsive and inclusive service delivery, \textit{dynamic public resource allocation} is emerging as a vital AI-enabled mechanism for optimizing operational efficiency under complex uncertainties, such as sudden traffic disruptions or spontaneous crowd gatherings~\cite{Liu2022intro}. 
Yet, as shown in \textbf{Fig.~\ref{fig:example}}, 
real-world deployment remains highly challenging. 
Because individual resources are typically limited in their service capacity, and public demands are inherently volatile and spatially clustered, relying on independent dispatches is often inadequate, so that explicit coordination among multiple resources to collectively serve a high-demand region becomes strictly necessary to prevent localized system overload~\cite{Lv2024intro}. These factors underscore the urgency of developing principled and scalable approaches that can reason about multi-agent cooperation, adapt to rapidly changing environmental conditions, and improve social outcomes in a wider range of applications.

\begin{figure}[!t] 
    \centering
    \includegraphics[width=0.46\textwidth]{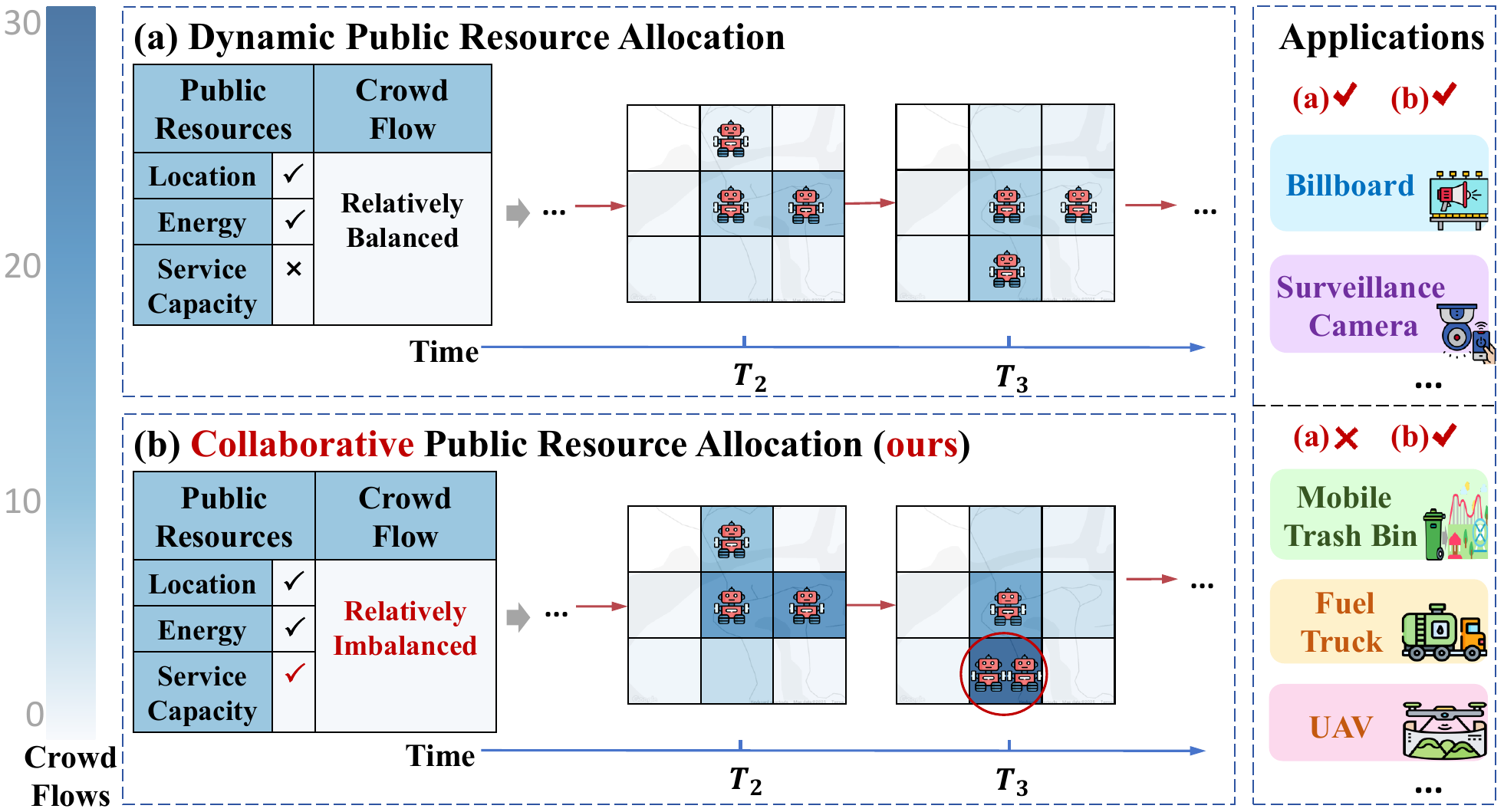}
    \caption{Dynamic public resource allocation vs. Collaborative public resource allocation (CPRA).}
    \label{fig:example}
    \vspace{-1em}
\end{figure}


Unfortunately, most existing AI-based approaches to dynamic allocation treat each resource independently, neglecting the vital alignment between individual actions and global system objectives~\cite{Ruan2020pra}. 
This oversimplification fails to capture the collaborative nature of real-world urban systems~\cite{wu2024mas}, where multiple limited-capacity resources must jointly and strategically serve emerging high-demand areas~\cite{Sadeghi2022intro}. For example, during peak traffic hours or large-scale emergencies, a single ambulance or mobile testing unit is rarely sufficient to handle the overwhelming influx of local demands, inevitably leading to prolonged response latencies and unmet public needs~\cite{Ji2019pra}. This critical gap reveals the fundamental requirement for effective dynamic allocation: the absolute necessity for \textit{constraint-aware cooperative decision-making} under highly volatile \textit{spatio-temporal dynamics}.


To bridge this gap, we introduce the \underline{C}ollaborative \underline{P}ublic \underline{R}esource \underline{A}llocation (CPRA) problem, as shown in \textbf{Fig.~\ref{fig:example}}. 
CPRA explicitly models system-level coordination under practical constraints (e.g., limited service capacity) and spatio-temporally evolving demands. 
By integrating collaborative capacity awareness directly into the formulation, CPRA fundamentally shifts the paradigm from independent planning to cooperative coalition formation. This ensures that emerging high-demand areas are served by an optimal group of resources, effectively preventing both localized service overload and redundant deployments. 
However, compared to dynamic public resource allocation, CPRA involves a fundamentally more complex state and action space, presenting two core challenges:

\par \textbf{1) Combinatorial complexity in collaborative scheduling.} 
Enabling effective collaborative scheduling among all resources is highly challenging due to the explosion of the joint action and state spaces as the number of public resources scales. 
Traditional Multi-Agent Reinforcement Learning (MARL) methods often struggle in this highly constrained setting~\cite{Li2022bas}. 
Specifically, when multiple resources collaborate to serve a shared high-demand region, it becomes notoriously difficult to accurately evaluate each individual resource's contribution to the global utility, which is known as the credit assignment problem~\cite{wu2024mas}. 
Consequently, existing methods easily fall into suboptimal solutions, either failing to converge efficiently during training or yielding uncoordinated allocation strategies that inadvertently violate capacity constraints.

\textbf{2)Extreme dynamics in spatio-temporal demands.} 
In CPRA, public demands exhibit highly volatile, non-stationary, and complex spatio-temporal patterns~\cite{wu2025mob}. 
Existing heuristic algorithms are inherently myopic, they passively react to immediate localized demands and fail to optimize long-term system utility. 
On the other hand, purely predictive spatio-temporal models, while capable of forecasting future trends, typically lack the mechanism to integrate these data-driven priors into the reinforcement learning environment~\cite{Tan2024sim}~\cite{Gao2024bas}. 
Therefore, they struggle to translate the extracted spatio-temporal dependencies into proactive, high-coverage decision-making. Bridging the gap between extracting complex spatio-temporal dynamics and directly utilizing them to guide collaborative dispatching remains a significant hurdle.

To address these challenges, we propose a new framework called \underline{G}ame-Theoretic \underline{S}patio-\underline{T}emporal \underline{R}einforcement \underline{L}earning (GSTRL) for solving CPRA. Targeting the first challenge, we prove that the CPRA problem can be formulated as a \emph{Potential Game}, and then guide the reward learning rule for the system based on the potential function,
providing theoretical guarantees for approximating the Nash equilibrium of this NP-hard problem. To tackle the second challenge, GSTRL integrates spatio-temporal feature extraction~\cite{liu2024ref} and adaptive policy learning for multiple public resources and demands, enabling effective dispatch strategies under real-world constraints.
Our major contributions can be summarized as follows:
\vspace{-1em}
\begin{itemize}[leftmargin=*]

   \item \textit{A novel perspective from the game theory}: To the best of our knowledge, we provide the first game-theoretic formulation for CPRA, considering  \textit{agent collaboration} under service capacity constraints.
   By formulating it as a potential game with a constructed potential function to capture collaboration dynamics, we provide theoretical guarantees for accelerating convergence towards the Nash equilibrium.

    \item \textit{A new framework for CPRA}: We propose GSTRL, a reinforcement learning (RL) framework to solve the CPRA problem.
    By extracting the spatio-temporal features of crowd flow distribution, adaptively fusing features from public resources and crowd flows, and incorporating masking mechanisms, GSTRL effectively makes decisions based on the system’s \textit{spatio-temporal dynamics} while strictly adhering to environmental constraints.

    \item \textit{Comprehensive empirical evidence}: Extensive experiments on two real-world datasets show that GSTRL outperforms state-of-the-art baselines by up to 40\%, demonstrating its superior performance on the CPRA problem across various parameter settings. Ablation and spatio-temporal dynamics experiments further highlight the role of its key components in enhancing learning effectiveness.

\end{itemize}

\section{Overview}


In this section, we give some preliminaries, then we formally define CPRA and prove its NP-hardness for the first time.

\subsection{Preliminaries}

\begin{definition}[Location]
We uniformly partition the area of interest into $N$ grid locations, denoted by $\mathcal{G} = \left\{ g_i \right\}$.
\end{definition}

\begin{definition}[Public Resource]
At any time interval $T_t$, the resource $\mathbf{m}_k$ is represented as a triple $\mathbf{m}_t =\left( l_{k}^{t}, e_{k}^{t}, p_{k}^{t} \right)$, where $l_{k}^{t} \in \mathcal{G}$ denotes its current location, $e_{k}^{t}$ represents its remaining energy, and $p_{k}^{t}$ indicates its service capacity.
\end{definition}

\begin{definition}[Energy Cost]
When a resource moves from $g_i$ to $g_j$, the energy cost is denoted by $c_{i,j}$. We treat all resources of the same type, making cost solely dependent on Euclidean distance~\cite{Ruan2020pra}.
\end{definition}

\begin{definition}[Service Capacity]
The service capacity of a resource $k$ at time $t$, denoted by $p_{k}^{t}$, measures the crowd flow that the resource can serve at time $t$. We assume that the service capacity is identical for all resources and is updated at each time step, such that $p_{k}^{t+1} = p_{k}^{t} = p_{k+1}^{t}$. 
\end{definition}

\begin{definition}[Resource Depot]
A resource depot $g_{\omega} \in \mathcal{G}$ is static. All public resources must start and end their journey at this resource depot during the scheduling process.
\end{definition}

\begin{definition}[Crowd Coverage]
The crowd coverage during the time interval $T_t$ is defined as the total number of people successfully served by all public resources across the entire area. Let $u_{i,k}^{t}$ indicate the presence of public resource $m_k$ in grid cell $i$ at time $t$, and let $\lambda_{i}^{t}$ represent the remaining crowd flow in grid $i$ at time $t$. Consequently, the crowd coverage during the time interval $T_t$ can be expressed as:
{\footnotesize
\begin{equation}
  C_t = \sum_{k \in \mathcal{K}} \sum_{g_i \in G} \left( u_{i,k}^{t} \cdot \min \left\{ \lambda_{i}^{t}, p_{k}^{t} \right\} \right).
    \label{coverage}
\end{equation}
}
\end{definition}

\subsection{Problem Statement}

Given a set of resources $\mathcal{K} = \left\{ k = 1, \cdots, k_{\max} \right\}$, a set of time intervals $\mathcal{T} = \left\{ T_t \mid t = 1, \cdots, t_{\max} \right\}$, an initial energy $E$, and the location of the resource depot $g_{\omega}$, our objective is to find a scheduling strategy that maximizes long-term coverage while adhering to energy limits and ensuring that all resources start and end their journeys at the resource depot. 
Since most resources operate in fixed positions, decision-making time dominates travel time, which is neglected.

Based on flow distributions, this optimization problem is formulated as an Integer Linear Programming (ILP) problem with respect to the decision variable $u_{i,k}^{t}$. To facilitate the solution, we introduce an auxiliary variable $x_{i,j,k}^{t}$, which indicates whether resource $m_k$ travels from location $g_i$ to location $g_j$ at the beginning of $T_t$:
{\footnotesize
\begin{equation}
\sum_{g_h \in G}{x_{h,i,k}^{t}} = \sum_{g_j \in G}{x_{i,j,k}^{t+1}} = u_{i,k}^{t},  \forall g_i \in \mathcal{G}, T_t \in \mathcal{T}, k \in \mathcal{K}.
    \label{variable}
\end{equation}
}

The formulation of CPRA is expressed as follows:
{\footnotesize
\begin{align}
 \mathcal{P}: \;& \underset{\left\{ u_{i,k}^{t} \right\}}{\max}\sum_{t\in \mathcal{T}}{\sum_{k\in \mathcal{K}}{\sum_{g_i\in \mathcal{G}}{\left( u_{i,k}^{t}\cdot \min \left\{ \lambda _{i}^{t}, p_{k}^{t} \right\} \right)}}}   \label{P0}\\
  s.t. \;\; 
  & C_1: \sum_{k\in \mathcal{K}}{\sum_{g_i\in \mathcal{G}\,\backslash\left\{ g_{\omega} \right\} ,g_j\in \mathcal{G}}{x_{i,j,k}^{t}}}=0, \tag{\ref{P0}{a}}\\
  & C_2: \sum_{k\in \mathcal{K}}{\sum_{g_j\in \mathcal{G}}{x_{\omega ,j,k}^{1}}}=\sum_{k\in \mathcal{K}}{\sum_{g_i\in \mathcal{G}}{x_{i,\omega ,k}^{n+1}}}=k_{\max}, \tag{\ref{P0}{b}}\\
  & C_3: \sum_{k\in \mathcal{K}}{u_{i,k}^{t}}\leq \frac{\sum_{g_i\in \mathcal{G}}{\lambda _{i}^{t}}}{p_{k}^{t}}, \forall g_i\in \mathcal{G}, T_t\in \mathcal{T},\tag{\ref{P0}{c}}\\
  & C_4: \sum_{t\in \mathcal{T}\cup \left\{ t_{n+1} \right\}}{\sum_{g_i,g_j\in \mathcal{G}}{c_{i,j}\cdot x_{i,j,k}^{t}}}\leq E, \forall k\in \mathcal{K},\tag{\ref{P0}{d}}
\end{align}
}

\noindent where constraint $C_1$ ensures all routes start from the resource depot $g_{\omega}$. $C_2$ enforces that each resource starts and ends at $g_{\omega}$, occupying a location at both $T_1$ and $T_{t_{\max}}$. $C_3$ limits the number of resources in $g_i$ at time $T_t$, preventing trivial stationary solutions. $C_4$ ensures each resource has sufficient energy $E$ to return to the depot.

\begin{theorem}
The CPRA problem $\mathcal{P}$ is NP-hard.
\end{theorem}

\begin{proof}
To prove that $\mathcal{P}$ is NP-hard, we reduce a NP-hard problem, the Multiple Knapsack Problem (MKP), to a case of $\mathcal{P}$.

Consider a special case of problem $\mathcal{P}$ where:  
a) There is a single time interval, i.e., $|T| = 1$.  
b) Each grid cell $g_i \in G$ contains exactly one unit of crowd flow $\lambda_{i}^{t}$.  
c) Each public resource $m_k \in K$ corresponds to a knapsack.  
d) The service capacity $p_{k}^{t}$ of each public resource corresponds to the capacity of each knapsack.  
e) The decision variable $u_{i,k}^{t}$, which indicates whether the resource $m_k$ serves the grid cell $g_i$, corresponds to placing a task in the knapsack.  

Under these conditions, problem $\mathcal{P}$ reduces to selecting a subset of tasks to maximize the objective $\sum_{i \in G} \min \left\{ \lambda_{i}^{t}, p_{k}^{t} \right\}$ while satisfying the capacity constraint $\sum_{i \in G} u_{i,k}^{t} \leq p_{k}^{t}$. This mirrors the structure of the Multiple Knapsack Problem, where the goal is to maximize the value of selected items under capacity constraints.

Since the Multiple Knapsack Problem is NP-hard, and the special case of $\mathcal{P}$ is equivalent to MKP, it follows that $\mathcal{P}$ is at least as hard as the Multiple Knapsack Problem. Hence, $\mathcal{P}$ is NP-hard. 
\end{proof}

We demonstrate that CPRA $\mathcal{P}$ is computationally intractable, underscoring the need for advanced algorithmic approaches.

\section{A Game-Theoretic View on CPRA}

To design an effective algorithm for the NP-hard CPRA problem, we analyze the relationship between individual agents and the overall system based on Potential game theory. 

\begin{figure*}[!ht]
\centering
\includegraphics[width=0.9\textwidth]{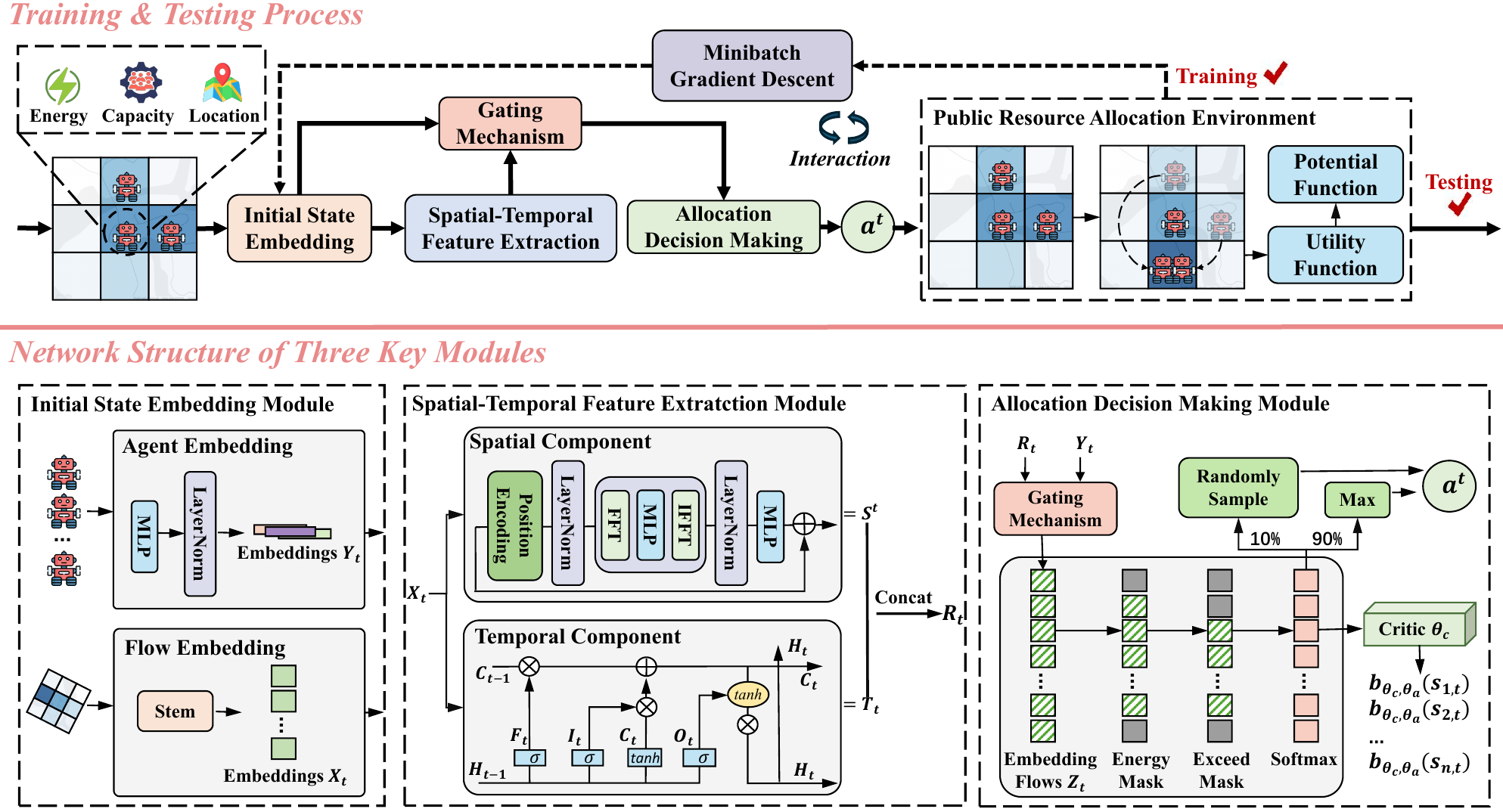}
\vspace{-0.5em}
\caption{Overall architecture of the proposed GSTRL framework.}\label{sua}
\end{figure*}

\subsection{Potential Game Model}
\label{game_theory}

We treat each public resource as a player and the combination of resources and crowd flow as a system. The CPRA problem is modeled as a game $ G = \{ \mathcal{K}, \mathcal{S}, \{U_k\}_{k \in \mathcal{K}} \} $:
\begin{itemize}[leftmargin=*]
    \item $ \mathcal{K} $ is the set of players (public resources).
    \item $ \mathcal{S} = S_1 \times S_2 \times \dots \times S_{k_{max}} $ represents the joint strategy space of the game, with $ S_k $ being the set of all possible strategies for public resource $ m_k $.
    \item $ U_k $ is the utility function of public resource $m_k$ .
\end{itemize}
Each strategy profile $ \zeta \in \mathcal{S} $ is defined as $ \zeta = (\zeta_1, \zeta_2, \dots, \zeta_{k_{max}}) $, which can also be expressed as $ \zeta = (\zeta_k, \zeta_{-k}) $, where $ \zeta_{-k} $ represents the strategies of all resources except $ m_k $. The goal of each resource is to maximize its utility $ U_k(\zeta) $ based on the allocation strategy.

\begin{definition}[Utility function of public resource $ m_k $]
The utility of public resource $ m_k $ under strategy $ \zeta $ is defined as the sum of the crowd coverage achieved at time $ t $:
\begin{equation}
    U_k(\zeta )=\sum_{g_i\in \mathcal{G}}{\left( u_{i,k}^{t}\cdot \min \left\{ \lambda _{i}^{t},p_{k}^{t} \right\} \right)}.
\end{equation}
\end{definition}

\begin{definition}[Potential Game]
A game $ G = \{ \mathcal{K}, \mathcal{S}, \{U_k\}_{k \in \mathcal{K}} \} $ is a \textbf{potential game} if there exists a potential function $\Phi :\mathcal{S}\rightarrow \mathbb{R}$ such that for all $ k \in \mathcal{K} $, all $ \zeta_{-k} \in S_{-k} $, and all $ \zeta_k, \zeta_k' \in S_k $,
\[
U_k(\zeta_k, \zeta_{-k}) - U_k(\zeta_k', \zeta_{-k}) = \Phi(\zeta_k, \zeta_{-k}) - \Phi(\zeta_k', \zeta_{-k}).
\]
\end{definition}

\begin{theorem}
The game $ G = \{ \mathcal{K}, \mathcal{S}, \{U_k\}_{k \in \mathcal{K}} \} $ is a potential game with the potential function:
\begin{equation}
\Phi(\zeta) = \sum_{k\in \mathcal{K}}{\sum_{g_i\in \mathcal{G}}{\left( u_{i,k}^{t}\cdot \min \left\{ \lambda _{i}^{t},p_{k}^{t} \right\} \right)}}.
\end{equation}
\end{theorem}

\begin{proof}
The potential function $ \Phi(\zeta) $ is defined as the total crowd coverage achieved by all public resources:
\[
\Phi(\zeta) = \sum_{k \in \mathcal{K}} U_k(\zeta) = \sum_{k\in \mathcal{K}}{\sum_{g_i\in \mathcal{G}}{\left( u_{i,k}^{t}\cdot \min \left\{ \lambda _{i}^{t},p_{k}^{t} \right\} \right)}}.
\]

\noindent Now, consider the change in the strategy of a single resource $ m_k $ from $ \zeta_k $ to $ \zeta_k' $, while all other strategies $ \zeta_{-k} $ remain fixed. The resulting change in the potential function is:
\[
\Phi(\zeta_k, \zeta_{-k}) - \Phi(\zeta_k', \zeta_{-k}) = U_k(\zeta_k, \zeta_{-k}) - U_k(\zeta_k', \zeta_{-k}).
\]
This confirms that the difference in the potential function directly reflects the change in the utility of the public resource $ m_k $, satisfying the conditions for a potential game.
\end{proof}

\begin{theorem} 
\label{theorem3}
Every finite potential game has at least one pure-strategy Nash equilibrium.
\end{theorem}

\begin{proof}
Since $ G $ is a finite potential game and the potential function $ \Phi(\zeta) $ is bounded over the finite strategy space $ \mathcal{S} $, there exists a strategy profile $ \zeta^* $ that maximizes $ \Phi $:
\[
\zeta^* = \arg \max_{\zeta \in \mathcal{S}} \Phi(\zeta).
\]
At $ \zeta^* $, no resource $ m_k $ can change its strategy to increase its utility, as such a change would increase $ \Phi $, contradicting the minimality of $ \Phi(\zeta^*) $. Therefore, $ \zeta^* $ is a pure-strategy Nash equilibrium.
\end{proof}

\subsection{Further Discussion}

Guided by \textbf{Theorem~\ref{theorem3}}, we employ an RL algorithm to address CPRA.
We adopt the potential function as the reward function, establishing theoretical guarantees for the algorithm's convergence and approximating the near-optimal solution.

For the CPRA problem, we further demonstrate that there exists no gap between the objective and the Nash equilibrium~\cite{cheng2025mat}. From a system-level perspective, this ensures that maximizing each player's utility inherently maximizes overall system performance~\cite{Simone2019mat}, leading to no resource redundancy. From a decision-making perspective, this alignment implies that information sharing does not introduce bias~\cite{Ding2025coo}. Motivated by these insights, we adopt a centralized decision-making framework in which the entire system learns an optimal strategy under a unified reward structure.

\section{Methodology}

\subsection{Framework Overview}

Given a distribution of crowd flows and resources, our goal is to maximize long-term crowd coverage by placing resources in the optimal locations at each time step. \textbf{Fig.~\ref{sua}} illustrates the framework, which mainly consists of three modules:
\vspace{-0.3em}
\begin{itemize}[leftmargin=*]
    \item \textbf{Initial State Embedding}: This module is divided into two components: crowd flow embedding and agent embedding. First, we flatten the predicted flows and use an MLP layer to process the initial state of public resources~\cite{wang2025air}. Then, a Stem and LayerNorm component is applied to extract high-level features (denoted as $\mathbf{X_t}$ and $\mathbf{Y_t}$), where \( t \in\{ t_1, t_2, \dots, t_{\text{max}}\} \).
    \item \textbf{Spatio-Temporal Feature Extraction}: Starting with the crowd flow embedding output \( X_t \), this module uses three Bidirectional LSTMs and a Fourier Neural Operator to capture the spatial and temporal dynamics of \( X_t \). The spatial feature \( S^t \) and temporal feature \( T_t \) are then concatenated as \( R_t \) for further processing~\cite{chen2025sel}.
    \item \textbf{Allocation Decision Making}: We first apply a gating mechanism to integrate the spatio-temporal embedding of flows \( R_t \) with the resources embedding \( Y_t \), followed by a mask criterion to ensure the strategies comply with practical constraints. A softmax layer is then used to calculate the score for each action \( Q_k \), which is executed accordingly.
\end{itemize}
\vspace{-0.3em}
During training, agents interact with the environment and get experience. We sample mini-batches from the experience and refine the strategy using mini-batch gradient descent until convergence~\cite{qi2023min}.

\subsection{CPRA Modeled as Markov Decision Process}

To capture long-term effects and avoid myopic decisions, we formulate the resource allocation process as a Markov Decision Process (MDP)~\cite{wang2025urb}. It consists of four components: states, actions, transitions, and rewards.

\noindent\textbf{a) State.}
\label{subsubsec:state}
We denote the state at time $t$ by $s_{t} = (\mathbf{d}_{t}, \mathcal{M}_{t})$, where $\mathbf{d}_{t} = \{\lambda_{i}^{t}\}_{g_i \in \mathcal{G}}$ is the set of predicted flows in each grid, and 
$\mathcal{M}_{t} = \{(l_{k}^{t}, e_{k}^{t}, p_{k}^{t})\}_{k=1}^{k_{\max}}$ is the set of resources.

\noindent\textbf{b) Action.}
\label{subsubsec:action}  
At time step $t$, the action $a_{t} = \left( g_{i}^{t}, \mathbf{m}_{k}^{t}\right)$ means that we allocate resource $\mathbf{m}_{k}$ to grid $g_{i}$. We adopt a sequential assignment strategy within each time step. Specifically, once a resource is allocated, the remaining crowd flow of the target grid is immediately updated to reflect the fulfilled demand before executing the allocation for the subsequent resources.

\noindent\textbf{c) Transitions.}
\label{subsubsec:transitions}
Once an action $a_{t}$ is taken, the system deterministically transitions to $s_{t+1}$ based on $s_{t}$ and $a_{t}$.

\noindent\textbf{d) Reward.}
\label{subsubsec:reward}
Based on potential game theory, we adopt the potential function $\Phi(\zeta)$ as the reward function, which is
$ r_t \;=\; \sum_{k \in \mathcal{K}}\!\sum_{g_i \in \mathcal{G}} \bigl(u_{i,k}^{t} \cdot \min\{\lambda_i^{t}, p_{k}^{t}\}\bigr)$ at time $t$.
By using a global optimization function instead of individual utility functions for each resource, GSTRL can capture the agent collaboration dynamics more effectively so as to accelerate convergence toward a near-Nash equilibrium.

\vspace{-0.5em}
\subsection{Initial State Embedding }

We design an Initial State Embedding Module to convert predicted crowd flows and resource states into high-level features for subsequent learning. This involves two parts: Crowd Flow Embedding and Agent Embedding.

\vspace{-0.5em}
\subsubsection{Flow Embedding.}

Input data reflects the dynamics of crowd flows~\cite{liu2024per}. After flattening, we feed the data into a stem module with $1\times1$ convolutions to extract a feature map $\mathbf{X_t}$.

\subsubsection{Agent Embedding.}

Each resource is denoted by $(l_{k}^{t},\,e_{k}^{t},\,p_{k}^{t})$. We project the 3-dimensional vector into a higher-dimensional space via a fully connected layer, followed by LayerNorm. We denote the embedding resulting from each public resource by $\mathbf{Y_t}$.

\subsection{Spatio-Temporal Feature Extraction}

Our module refines the features from the Crowd Flow Embedding via two components: a Temporal and a Spatial Component.

\subsubsection{Temporal Component.}

We adopt a BiLSTM to capture long-term temporal patterns in crowd flow. Unlike unidirectional models~\cite{liu2024timeffm}, BiLSTM processes each time step in both forward and backward directions. This process yields hidden states $\overrightarrow{\mathbf{H}}_t$ and $\overleftarrow{\mathbf{H}}_t$. Concatenating them produces a comprehensive temporal embedding $\mathbf{T}_t$, which helps avoid suboptimal dispatching~\cite{hu2024tis,hu2024att}.

\subsubsection{Spatial Component.}

We adopt FNO~\cite{Li2020FNO} to capture global correlations in the frequency domain with reduced complexity. Let $\mathbf{X}_t \in \mathbb{R}^{N \times E}$ denote the spatial feature map at time $t$. After position encoding and LayerNorm, we transform $\mathbf{X}_t$ via FFT, multiply by a learnable weight matrix $\hat{\mathbf{W}}$, and apply the inverse FFT:
\begin{equation}
\mathbf{s}^{t}_{i} = 
\mathrm{IFFT}\bigl(\hat{\mathbf{W}}\cdot\mathrm{FFT}(\mathbf{X}_t)\bigr)_{i},
\end{equation}

\noindent where $i$ indexes the grids, and the structure of $\hat{W}$ reduces parameter count. This yields spatial embeddings $\mathbf{S}^t = \{\mathbf{s}^{t}_{1}, \dots, \mathbf{s}^{t}_{N}\}$ in $\mathcal{O}(NE \log N)$ time, where $N$ is the number of grids and $E$ the embedding dimension.

\subsubsection{Output.}
We concatenate \(\mathbf{T}_{t}\) and \(\mathbf{S}^{t}\) to form the spatio-temporal embedding \(\mathbf{R}_t = \mathrm{Concat}(\mathbf{T}_{t},\, \mathbf{S}^{t})\), which captures both time-series trends and cross-grid interactions. \(\mathbf{R}_t\) is then passed to the Allocation Decision Making Module, enabling resource allocation under capacity constraints and fluctuating demands~\cite{liang2021tra}~\cite{wu2025com}.

\subsection{Allocation Decision Making}

This module integrates the spatio-temporal embedding $R_t$ and the resource embedding $Y_t$ via a gating mechanism:
\begin{equation}\label{eq:Jt}
 \mathbf{J}_t \;=\; \sigma\bigl(\mathbf{W}_j \cdot [\,\mathbf{R}_t,\, \mathbf{Y}_t\,] + \mathbf{b}_j\bigr), \mathbf{Z}_t \;=\; \mathbf{J}_t \cdot \mathbf{R}_t \;+\; \bigl(1 - \mathbf{J}_t\bigr)\,\cdot\, \mathbf{Y}_t,
\end{equation}

\noindent where $\mathbf{W}_j$ and $\mathbf{b}_j$ are learnable parameters, and $\sigma(\cdot)$ is the sigmoid function. The resulting $\mathbf{Z}_t$ represents the overall feature for the system for subsequent decision-making.

When making decisions, the four constraints $C_1 \text{--} C_4$  must be satisfied. Although $C_1$ and $C_2$ are effective by design, $C_3$ and $C_4$ require restricting locations. 
We thus apply masking mechanisms to $\mathbf{Z}_t$, yielding a masked representation $\overline{\mathbf{Z}_t}$.

Finally, We decode $\overline{\mathbf{Z}_t}$ through a softmax layer
$\mathbf{Q}_k \;=\; \mathrm{Softmax}\bigl(\,\overline{\mathbf{Z}_t}\bigr)$. During training, the agent chooses the highest-probability action $90\%$ of the time and randomly explores $10\%$ of the time. The processed state is also passed to the Critic $b_{\theta_c,\theta_a}(s_{i,t})$ for estimation.

\subsection{Training Process}

To learn an effective allocation strategy, we adopt an Actor-Critic training framework~\cite{mao2023drl}. The procedural flow of the training process is depicted 
in \textbf{Algorithm~\ref{alg:algorithm}}.

\begin{algorithm}[ht!]
\caption{GSTRL Algorithm}
\label{alg:algorithm}
\raggedright
\textbf{Input}: Input sets of each public resource's state $M_t$, and sets of predicted crowd flow's amount $G_t$. The actor network with parameters $\theta _a$, and the critic network with parameters $\theta _c$. \\
\textbf{Output}: Trained actor network parameter $\theta _{a}^{*}$.
\par
\begin{algorithmic}[1] 
    \STATE Randomly initialize $\theta _a$ and $\theta _c$.
    \STATE Initialize replay buffer $\mathcal{B}$.
    \FOR{episodes $k=0,1,2,...,EPI _{max}$}
    \STATE Generate the action $a_t$ according to $\pi _{\theta _{\alpha}}$.
    \STATE Execute the action $a_t$ and obtain $s_t$ and $r_t$.
    \STATE Store $\left( s_t, a_t, s_{t+1}, r_t \right)$ into replay buffer $\mathcal{B}$.
    \STATE Randomly sample $M$ samples from the replay buffer of length $N$.  
    \FOR{$t=0,1,2,...,t_{max}$}
    \FOR{$k=0,1,2,...,k_{max}$}
    \STATE Calculate the expected reward.
    \STATE Compute the Q-value $Q_{\pi_{\theta_a}}(s_{t},a_{t})$ based on the samples.
    \STATE Compute Advantage $L\left( \omega \right)$ based on Eq.~\eqref{b-q}.
    \ENDFOR
    \ENDFOR
    \STATE Compute the loss of critic $\mathcal{L}_{\theta_{c}}$ based on Eq.~\eqref{e-q}. 
    \STATE Compute the loss of policy $\mathcal{L}_{\theta_{a}}$ based on Eq.~\eqref{f-q}.
    \STATE Calculate the joint loss $\mathcal{L}_{\theta_{ac}}$ based on Eq.~\eqref{g-q}.
    \STATE Update the parameters of the actor-critic model.  
    \ENDFOR
    \STATE \textbf{return} $\theta _{a}^{*} = \theta _a$.
\end{algorithmic}
\end{algorithm}

\subsubsection{Actor-Critic Overview.}
Our Actor-Critic method comprises two main components: 
\begin{itemize}[leftmargin=*]
    \item \textbf{Actor (Policy Network):} The policy network, denoted as $\pi_{\theta_a}$, governs resource allocation decisions of GSTRL. 
    The Initial State Embedding Module, Spatio-Temporal Feature Extraction Module, and Allocation Decision Module are essential components of the policy network.
    \item \textbf{Critic (Value Network):} The value network approximates the value function \(V_{\pi_{\theta_a}}(s_{t})\), which evaluates state quality and guides the Actor towards higher-value actions. Implemented as a feed-forward network with a regression layer and parameterized by \(\theta_{c}\), the value network is trained alongside the Actor network but is not used after training.
\end{itemize}
By integrating these two components, GSTRL dynamically optimizes allocation strategies using mini-batch gradient descent.

\subsubsection{Value Functions and Advantage.}
Given a policy $\pi_{\theta_a}$, we define the state-action function $Q_{\pi_{\theta_a}}(s_{t},a_{t})$ and the state-value function $V_{\pi_{\theta_a}}(s_{t})$ as:
\begin{align}
Q_{\pi_{\theta_a}}(s_{t},a_{t}) &= \mathbb{E}_{\pi_{\theta_a}}\bigl[r_{t} \,\bigm|\, s=s_{t}, a=a_{t}\bigr], \\
V_{\pi_{\theta_a}}(s_{t}) &= \mathbb{E}_{a_{t}\sim\pi_{\theta_a}(s_{t})}\bigl[\,Q_{\pi_{\theta_a}}(s_{t}, a_{t})\bigr]. 
\end{align}
The estimated value function is subsequently used to compute the advantage function $A_{\pi_{\theta_a}}(s_{t},a_{t})$:
\begin{equation}
A_{\pi_{\theta_a}}(s_{t},a_{t}) \;=\; Q_{\pi_{\theta_a}}(s_{t},a_{t}) \;-\; V_{\pi_{\theta_a}}(s_{t}).
\end{equation}
The advantage function reflects how much better (or worse) taking action $a_{t}$ in state $s_{t}$ is compared to the average action. To expedite computation, we approximate it as follows~\cite{Keneshloo2019DRL}:
\begin{equation}\label{b-q}
A_{\pi_{\theta_a}}(s_{t},a_{t}) \;\approx\; r_{t}\;+\;\gamma\,V_{\pi_{\theta_a}}(s_{t+1}) \;-\; V_{\pi_{\theta_a}}(s_{t}).
\end{equation}

\subsubsection{Loss Functions.}
The Critic is optimized via:
\begin{equation}\label{e-q}
\mathcal{L}_{\theta_{c}}
=\frac{1}{N}\!\sum_{i=1}^{N}\,\sum_{t \in \mathcal{T}}
\mathrm{smoothL1}\bigl(b_{\theta_{c},\theta_{a}}(s_{i,t}) - r_{i,t}\bigr),
\end{equation}

\noindent where the predicted value $b_{\theta_{c}, \theta_{a}}(s_{i,t})$ is defined as a baseline to evaluate the value function.
The Actor is trained to maximize the expected advantage:
\begin{equation}\label{f-q}
\mathcal{L}_{\theta_a} 
=\frac{1}{N}\!\sum_{i=1}^{N}\,\sum_{t \in \mathcal{T}}
\log \pi_{\theta_a}(a_{i,t} \mid s_{i,t})\,
A_{\pi_{\theta_a}}(s_{i,t}, a_{i,t}),
\end{equation}

\noindent where $N$ denotes the number of sampled trajectories, $\pi_{\theta_a}$ is the policy, and $A_{\pi_{\theta_a}}(\cdot)$ is the advantage function.
We train the Actor and Critic jointly by combining their losses:
\begin{equation}\label{g-q}
\mathcal{L}_{\theta_{ac}} 
=\; \alpha\,\mathcal{L}_{\theta_a} 
+\;\beta\,\mathcal{L}_{\theta_{c}},
\end{equation}

\noindent where $\alpha$ and $\beta$ are hyperparameters controlling the trade-off between policy optimization and value estimation.

\section{Experiments}

In this section, we conduct experiments to address the following research questions:
\begin{itemize}[leftmargin=*]
    \item \textbf{RQ1:} How does GSTRL perform compared to heuristic algorithms in terms of maximizing crowd coverage?
    \item \textbf{RQ2:} How do the key components of GSTRL contribute to learning effective resource allocation strategies?
    \item \textbf{RQ3:} How do we choose the component for extracting spatio-temporal dynamics from crowd flow?
    \item \textbf{RQ4:} How do the different settings of each hyperparameter impact the overall performance of GSTRL?
    \item \textbf{RQ5:}How does the trained GSTRL framework devise practical scheduling strategies in real-world scenarios?
\end{itemize}

\subsection{Experimental Settings}
\subsubsection{Datasets.}

We utilize two datasets, Happy Valley and TaxiBJ, for our capacitated resource scheduling task.

\begin{itemize}[leftmargin=*]
    \item \textbf{Happy Valley}: This dataset, obtained from Ruan et al.\cite{Ruan2020pra}, provides hourly gridded crowd flow observations for Beijing Happy Valley, a theme park covering $1.25 \times 10^5 m²$. The data spans from January 1, 2018, to October 31, 2018. We divide the park into a $51 \times 108$ grid for fine-grained analysis. The dataset is split into training (60\%), validation (20\%), and test (20\%) sets.
    \item \textbf{TaxiBJ}: This dataset, published by Zhang et al.\cite{zhang2017deep}, records taxi flow data across Beijing from July 1, 2013, to March 31, 2016. We represent the city as a $32 \times 32$ grid to model the spatial distribution of taxi movements. The dataset is partitioned into 70\% for training, 20\% for validation, and 10\% for testing.
\end{itemize}

\subsubsection{Baselines \& Evaluation Metrics.}
 
\begin{table*}[!t]
  \centering
  \tabcolsep=1.7mm 
  \vspace{-1em}
    \resizebox{\textwidth}{!}{
    \begin{tabular}{c|c|l|c|cc|cc|cc|cc}
    \toprule
    \multicolumn{1}{c|}{} & \multicolumn{1}{c|}{\multirow{2}{*}{\textbf{Method}}} & \multicolumn{1}{l|}{\multirow{2}{*}{\textbf{Algorithms}}} & \multirow{2}{*}{\textbf{\#Param(K)}} & \multicolumn{2}{c|}{\textbf{E=30}} & \multicolumn{2}{c|}{\textbf{E=40}} & \multicolumn{2}{c|}{\textbf{E=50}} & \multicolumn{2}{c}{\textbf{E=60}} \\
    \cline{5-12}
    & & &  & \textbf{ADCC $\uparrow$} & \textbf{$\Delta$ $\uparrow$} & \textbf{ADCC $\uparrow$} & \textbf{$\Delta$ $\uparrow$} & \textbf{ADCC $\uparrow$} & \textbf{$\Delta$ $\uparrow$} & \textbf{ADCC $\uparrow$} & \textbf{$\Delta$ $\uparrow$} \\
    \hline
    \multirow{11}{*}{\rotatebox{90}{Happy Valley}} &\multirow{4}{*}{Heuristic} 
    & Static & - & 1,566 & -14\% & 1,673 & -20\% & 2,076 & -17\% & 2,106 & -13\% \\
    & & BB  & - & 1,717 & -6\% & 1,686 & -20\% & 2,128 & -15\% & 2,235 & -8\% \\
    & & MYOPIC  & - & 1,756 & -4\% & 1,789 & -15\% & 2,206 & -12\% & 2,289 & -6\% \\
    & & EADS  & - & \underline{1,831} & - & \underline{2,096} & - & \underline{2,493} & - & \underline{2,432} & - \\
    \cline{2-12}
    &\multirow{4}{*}{RL} 
    & REINFORCE  & 211 & 1,936{\tiny $\pm$8} & 6\% & 2,181{\tiny $\pm$3} & 4\% & 2,239{\tiny $\pm$5} & -10\% & 2,103{\tiny $\pm$10} & -14\% \\
    & & PPO & 220 & 1,956{\tiny $\pm$5} & 7\% & 2,203{\tiny $\pm$2} & 5\% & 2,336{\tiny $\pm$3} & -6\% & 2,167{\tiny $\pm$7} & -11\% \\
    & & SAC & 220 & 1,963{\tiny $\pm$5} & 7\% & 2,365{\tiny $\pm$3} & 13\% & 2,396{\tiny $\pm$3} & -4\% & 2,196{\tiny $\pm$6} & -14\% \\
    & & HEM & 20 & 1,657{\tiny $\pm$5} & -10\% & 1,896{\tiny $\pm$6} & -10\% & 2,136{\tiny $\pm$10} & -14\% & 2,176{\tiny $\pm$6} & -11\% \\
    \cline{2-12}
    &\multirow{2}{*}{MARL} 
    & QMIX & 3,976 & 1,958{\tiny $\pm$12} & 7\% & 2,209{\tiny $\pm$10} & 5\% & 2,306{\tiny $\pm$5} & -8\% & 2,303{\tiny $\pm$8}
    & -5\% \\
    & & DAE  & 5,743 & 1,982{\tiny $\pm$12} & 8\% & 2,218{\tiny $\pm$8} & 6\% & 2,325{\tiny $\pm$6} & -7\% & 2,316{\tiny $\pm$10} & -5\% \\
    \cline{2-12}\cline{2-12}
    \rowcolor{white} 
    &
& \cellcolor{background_gray} 
\textbf{GSTRL (ours)} & \cellcolor{background_gray} 
\textbf{216} & \cellcolor{background_gray} 
\textbf{2,298}{\tiny $\pm$16} & \cellcolor{background_gray} 
\textbf{26\%} & \cellcolor{background_gray} 
\textbf{2,530}{\tiny $\pm$5} & \cellcolor{background_gray} 
\textbf{20\%} & \cellcolor{background_gray} 
\textbf{2,585}{\tiny $\pm$8} & \cellcolor{background_gray} 
\textbf{4\%} & \cellcolor{background_gray} 
\textbf{2,517}{\tiny $\pm$23} & \cellcolor{background_gray} 
\textbf{3\%} \\
    \bottomrule
    \toprule
    \multicolumn{1}{c|}{} & \multicolumn{1}{c|}{\multirow{2}{*}{\textbf{Method}}} & \multicolumn{1}{l|}{\multirow{2}{*}{\textbf{Algorithms}}} & \multirow{2}{*}{\textbf{\#Param(K)}} & \multicolumn{2}{c|}{\textbf{E=30}} & \multicolumn{2}{c|}{\textbf{E=40}} & \multicolumn{2}{c|}{\textbf{E=50}} & \multicolumn{2}{c}{\textbf{E=60}} \\
    \cline{5-12}
    & & &  & \textbf{ADCC $\uparrow$} & \textbf{$\Delta$ $\uparrow$} & \textbf{ADCC $\uparrow$} & \textbf{$\Delta$ $\uparrow$} & \textbf{ADCC $\uparrow$} & \textbf{$\Delta$ $\uparrow$} & \textbf{ADCC $\uparrow$} & \textbf{$\Delta$ $\uparrow$} \\
    \hline
    \multirow{10}{*}{\rotatebox{90}{TaxiBJ}} &
    \multirow{3}{*}{Heuristic} 
    & Static & - & 19.3 & -92\% & 19.3 & -92\% & 19.3 & -92\% & 19.3 & -92\% \\
    & & MYOPIC  & - & 19.3 & -92\% & 19.3 & -92\% & 19.3 & -92\% & 19.3 & -92\% \\
    & & EADS  & - & \underline{258} & - & \underline{242} & - & \underline{238} & - & \underline{242} & - \\
    \cline{2-12}
    &\multirow{4}{*}{RL} 
    & REINFORCE  & 669 & 223{\tiny $\pm$6} & -14\% & 230{\tiny $\pm$2} & -15\% & 285{\tiny $\pm$5} & 20\% & 296{\tiny $\pm$6} & 22\% \\
    & & PPO & 691 & 229{\tiny $\pm$3} & -11\% & 241{\tiny $\pm$2} & -0.4\% & 293{\tiny $\pm$2} & 23\% & 302{\tiny $\pm$6} & 25\% \\
    & & SAC  & 697 & 230{\tiny $\pm$5} & -11\% & 246{\tiny $\pm$2} & 2\% & 295{\tiny $\pm$3} & 24\% & 305{\tiny $\pm$6} & 26\% \\
    & & HEM & 60 & 190{\tiny $\pm$3} & -26\% & 197{\tiny $\pm$2} & -19\% & 256{\tiny $\pm$6} & 8\% & 273{\tiny $\pm$5} & 13\% \\
    \cline{2-12}
    &\multirow{2}{*}{MARL} 
     & QMIX & 8,675 & 235{\tiny $\pm$8} & -9\% & 246{\tiny $\pm$5} & 2\% & 297{\tiny $\pm$7} & 25\% & 306{\tiny $\pm$10} & 26\% \\
    & & DAE  & 12,256 & 243{\tiny $\pm$9} & -6\% & 253{\tiny $\pm$6} & 5\% & 305{\tiny $\pm$8} & 28\% & 316{\tiny $\pm$12} & 31\% \\
    \cline{2-12}
    \rowcolor{white} 
    &
& \cellcolor{background_gray} 
\textbf{GSTRL (ours)} 
& \cellcolor{background_gray} 
\textbf{685} 
& \cellcolor{background_gray} 
\textbf{259}{\tiny $\pm$2} 
& \cellcolor{background_gray} 
\textbf{0.4\%} 
& \cellcolor{background_gray} 
\textbf{260}{\tiny $\pm$5} 
& \cellcolor{background_gray} 
\textbf{7\%} 
& \cellcolor{background_gray} 
\textbf{313}{\tiny $\pm$2} 
& \cellcolor{background_gray} 
\textbf{32\%} 
& \cellcolor{background_gray} 
\textbf{338}{\tiny $\pm$6} 
& \cellcolor{background_gray} 
\textbf{40\%} \\
    \bottomrule
    \end{tabular}
}
\caption{Model comparison on the Happy Valley and TaxiBJ datasets. The $\Delta$ represents the improvements in ADCC compared to the EADS approach. Bold and underlined digits are the best and EADS approach values.}
  \label{tab:results}
  \vspace{-1em}
\end{table*}

To evaluate the CPRA problem, we implement several heuristic algorithms as baselines.
Besides, as we are the first to study CPRA, we adapt state-of-the-art RL algorithms from related domains for a comprehensive comparison.


\noindent \textbf{Heuristic Algorithms:}
\begin{itemize}[leftmargin=*]
    \item \textbf{Static:} An approach that allocates resources to the location with the highest crowd flow within the resources' capacity at the first time step, keeping the positions fixed thereafter.
    \item \textbf{B\&B~\cite{Lawler1966opt}:} Utilizes a branch-and-bound approach with the PuLP ILP solver to determine the optimal solution for the ILP problem.
    \item \textbf{MYOPIC~\cite{Ruan2020pra}:} An aggressive allocation strategy that makes scheduling decisions based solely on the predicted crowd flow for the next time step, using a greedy approach. Resources execute local optimal strategies within the initial energy.
    \item \textbf{EADS~\cite{Ruan2020pra}:} A two-step algorithm that schedules resources to maximize crowd coverage within agents' service times, serving as the state-of-the-art baseline among heuristic algorithms for this task. The selection criterion is adjusted to maximize the minimum of crowd flow and capacity, with the crowd flow updated after each agent's service. 
\end{itemize}

\noindent \textbf{RL Algorithms:}
\begin{itemize}[leftmargin=*]
    \item \textbf{REINFORCE~\cite{Zhang2020bas}:} 
     By mapping the extracted spatio-temporal features of crowd demands and resource states into a probability distribution over the dispatching action space, it updates the policy parameters using the cumulative rewards, serving as a fundamental baseline for policy-based methods.
    \item \textbf{PPO~\cite{Zhu2025rlbas}:} 
    As an advanced on-policy Actor-Critic framework, this baseline utilizes a clipped surrogate objective to constrain the policy update step size. It estimates the long-term state values of public resources and adaptively updates scheduling policies to effectively cope with volatile, non-stationary crowd fluctuations.
    \item \textbf{SAC~\cite{Zheng2024rlbas}:} 
    This baseline implements an off-policy maximum entropy Actor-Critic algorithm, encourages robust exploration within the vast and complex joint action space of CPRA, effectively preventing the policy from trapping into local optima while maintaining high sample efficiency during training.
    \item \textbf{HEM~\cite{Wang2024rlbas}:} 
    This baseline leverages a hierarchical reinforcement learning framework. It decomposes the allocation space into a macro-level manager for regional resource targeting and a micro-level worker for capacity-aware resource dispatching, striking a balance between global coordination and localized efficiency.
\end{itemize}

\noindent \textbf{MARL Algorithms:}
\begin{itemize}[leftmargin=*]
    \item \textbf{QMIX~\cite{Tabish2020bas}:} 
     By imposing a monotonicity constraint on the relationship between individual utilities and the total action-value function, QMIX allows each public resource to make localized dispatching decisions while guaranteeing that their joint actions maximize the global potential function, effectively addressing the multi-agent credit assignment problem.
    
    \item \textbf{DAE~\cite{Li2022bas}:} 
     By explicitly decomposing the joint action space and modeling the localized spatio-temporal dependencies among neighboring agents, DAE successfully handles the combinatorial explosion of the resources' action variables under strict service capacity constraints, thereby preventing resource conflicts or redundant clustering in high-demand zones.
\end{itemize}


\noindent \textbf{Evaluation Metrics:}
We follow \cite{Ruan2020pra} and use the average daily crowd coverage (ADCC) to assess model performance. Each method is run 5 times, and we report the average ADCC value along with the variance. The notation $\Delta$ indicates the relative improvement of ADCC over the EADS approach under varying initial energy $E$. 
ADCC is calculated as follows:
\begin{equation}
ADCC = \frac{\sum_{d \in \mathcal{D}} \sum_{t \in \mathcal{T}} C_t}{|\mathcal{D}|}.
\end{equation}

\vspace{-0.5em}
\subsubsection{Implementation Details.}


The baselines are implemented in Python, with training using PyTorch 2.1.3 on a single NVIDIA GeForce RTX 3090 GPU. The Adam optimizer is used with a batch size of 16, an initial learning rate of 0.05, and a hidden size of 32. Each agent's capacity is set to 10. The resource number is set to 10 and the energy is explored within the range of 30 to 60.

\subsection{Model Comparison (RQ1)}

In this section, we compare the performance of GSTRL with existing baselines in terms of ADCC, using two real-world datasets as shown in \textbf{Table \ref{tab:results}}. Our model consistently outperforms all baselines under varying initial energy budgets $E$, achieving an improvement of 3\% to 26\% in ADCC on the Happy Valley dataset, and 0.4\% to 40\% on the TaxiBJ dataset. This significant margin demonstrates the effectiveness of our method in tackling the highly constrained CPRA problem across scenarios of different scales. Specifically, while traditional heuristic and RL baselines often struggle with the joint action space and constraint limits, GSTRL successfully overcomes these hurdles. By incorporating potential game theory, our framework explicitly aligns individual resource dispatching behaviors with the system coverage, thereby avoiding uncoordinated assignments and consistently approaching the near-optimum.

As the initial energy $E$ increases, we observe a upward trend in ADCC for all algorithms across both datasets. A higher energy budget inherently expands the feasible exploration space for the public resources. With relaxed energy constraints, agents are no longer restricted to purely localized servicing, instead, they can execute longer-range cooperative dispatching to collectively serve distant high-demand clusters without the immediate risk of premature battery depletion. Consequently, this enlarged action space allows both heuristic and RL methods to discover more flexible allocation solutions that yield higher overall utility.

Additionally, for the Happy Valley dataset, we notice that when $E$ reaches a certain threshold (e.g., $E = 50$), the increase in ADCC noticeably slows down, and the performance gap between GSTRL and other baselines begins to narrow. This phenomenon can be attributed to the limited spatial scope of the Happy Valley region. Because the geographical bounding box is relatively small, an energy budget of $E = 50$ is highly abundant, meaning most individual resources already possess sufficient mobility to traverse and cover almost the entire area. In such states, the spatial and energy constraints of CPRA are significantly weakened, allowing even suboptimal baselines to achieve satisfactory coverage, which inevitably leads to a performance saturation.

In contrast, the dynamics within the TaxiBJ dataset are markedly different. While most RL methods continue to show significant improvement as $E$ increases, heuristic algorithms perform poorly, notably, MYOPIC completely degenerates into the Static baseline, and EADS's performance paradoxically declines with higher energy budgets. This divergence is fundamentally due to the larger spatial scale and greater spatio-temporal complexity of the TaxiBJ dataset. In such a vast and non-stationary environment, heuristic algorithms lack predictive foresight and easily fall into local optima, such as redundantly dispatching multiple resources to a single hotspot while neglecting future demands elsewhere. Conversely, RL methods, particularly GSTRL with its advanced spatio-temporal feature extraction module, can effectively anticipate complex future flow shifts. This allows GSTRL to proactively and strategically distribute resources across the complex grid, maintaining superior exploration strategies and maximizing long-term coverage.



\begin{table}[!t]
\centering
\tabcolsep=1.7mm
\begin{tabular}{l||c|c|c|c|c|c}
\hline
{\multirow{2}{*}{\textbf{Variant}}} & \multicolumn{3}{c|}{\textbf{Happy Valley}} & \multicolumn{3}{c}{\textbf{TaxiBJ}} \\ \cline{2-7}
& \textbf{E=30} & \textbf{E=40} & \textbf{E=50} & \textbf{E=30} & \textbf{E=40} & \textbf{E=50} \\ \hline \hline
w/o AE                       & 2,016         & 2,021          & 2,121  & 217 & 219 & 257          \\
w/o S                       & 1,272         & 1,372          & 1,385   & 139 & 141 & 167         \\
w/o T                       & 2,158         & 2,216          & 2,300    & 227 & 228 & 273        \\
w/o PF   & 1,982  & 2,046  & 2,165 & 225  &  213  & 262   \\
w/o EM                       & 136         & 368          & 551       & 89 & 106 &  130    \\ \rowcolor{gray!20}
\textbf{GSTRL}                      & \textbf{2,298}         & \textbf{2,530}          & \textbf{2,585}   & \textbf{259} & \textbf{260} & \textbf{313}   \\ 
\hline
\end{tabular}
\caption{Ablation Study Results.}
\label{tab:ab}
\vspace{-2em}
\end{table}

\subsection{Ablation Study (RQ2)}

To comprehensively assess the impact of each component on the performance of GSTRL, we perform ablation studies on two real-world datasets. These ablations validate our claims regarding the necessity of modeling spatio-temporal dynamics, reward alignment, and strict capacity constraints.

\subsubsection{Effects of Feature Learning.}

To evaluate the effectiveness of the modules responsible for representation learning, we conduct ablation studies by independently removing three key components: 
a) \textbf{Agent Embedding (w/o AE)}: the module that encodes agent characteristics (e.g., remaining energy, current location, and available service capacity) into the state representation. Without this, agents become entirely blind to their physical limitations and those of their peers, leading to uncoordinated assignments where resources with insufficient battery or depleted capacity are mistakenly dispatched to demanding regions. 
b) \textbf{Spatial Component (w/o S)}: the spatial feature extractor responsible for modeling the geographical correlations of crowd flows. Without this component, the model treats each grid independently, failing to capture the cross-regional influence of demands and resulting in spatially myopic deployments. 
c) \textbf{Temporal Component (w/o T)}: the module that captures the temporal dynamics of crowd flows. Its absence prevents the model from anticipating future demand trends, causing resources to passively react to dispersing crowds rather than proactively moving to emerging hotspots. 
As shown in \textbf{Table~\ref{tab:ab}}, the absence of any single component leads to a severe decline in overall performance, confirming their collective necessity in learning robust spatio-temporal representations for both the environment and the resources.

\subsubsection{Effects of Potential Function (w/o PF)}

We perform the ablation of the potential function by modifying the reward design, fundamentally changing it from our constructed potential-based reward to a naive maximum utility among all agents. As shown in \textbf{Table \ref{tab:ab}}, removing this component leads to a dramatic performance drop. This explicitly highlights the critical role of the potential function in solving the multi-agent credit assignment problem. By breaking the theoretical properties of the potential game, individual agent objectives are no longer perfectly aligned with the global system coverage. Consequently, the system degenerates into selfish competition, where resources greedily cluster around highly rewarding localized hotspots, while neglecting other regions, ultimately failing to approximate the Nash Equilibrium.

\subsubsection{Effects of Exceed Mask (w/o EM)}

The exceed mask is specifically designed as a regularizer derived directly from the service capacity constraints of CPRA. It actively prevents GSTRL from converging to suboptimal solutions, where a overwhelmingly large number of agents needlessly crowd into a single grid merely due to its high crowd flow density. As shown in \textbf{Table \ref{tab:ab}}, removing the exceed mask leads to a marked decrease in performance and training stability. The mask effectively serves as an inductive bias, ensuring that the agents form capacity-aware coalitions without wasting service availability on over-supplied regions.

\subsection{Spatio-Temporal Dynamics Analysis (RQ3)}

To evaluate the ability of the spatial-temporal feature extraction module, 
we design five GSTRL variants by replacing this module. These include state-of-the-art models for spatio-temporal prediction, such as SIMVP~\cite{Tan2024sim}, which excels on grid-based data, and STDMAE~\cite{Gao2024bas}, which is designed for graph-structured data.

\begin{figure}[!t]  
  \centering
  \includegraphics[width=0.95\linewidth]{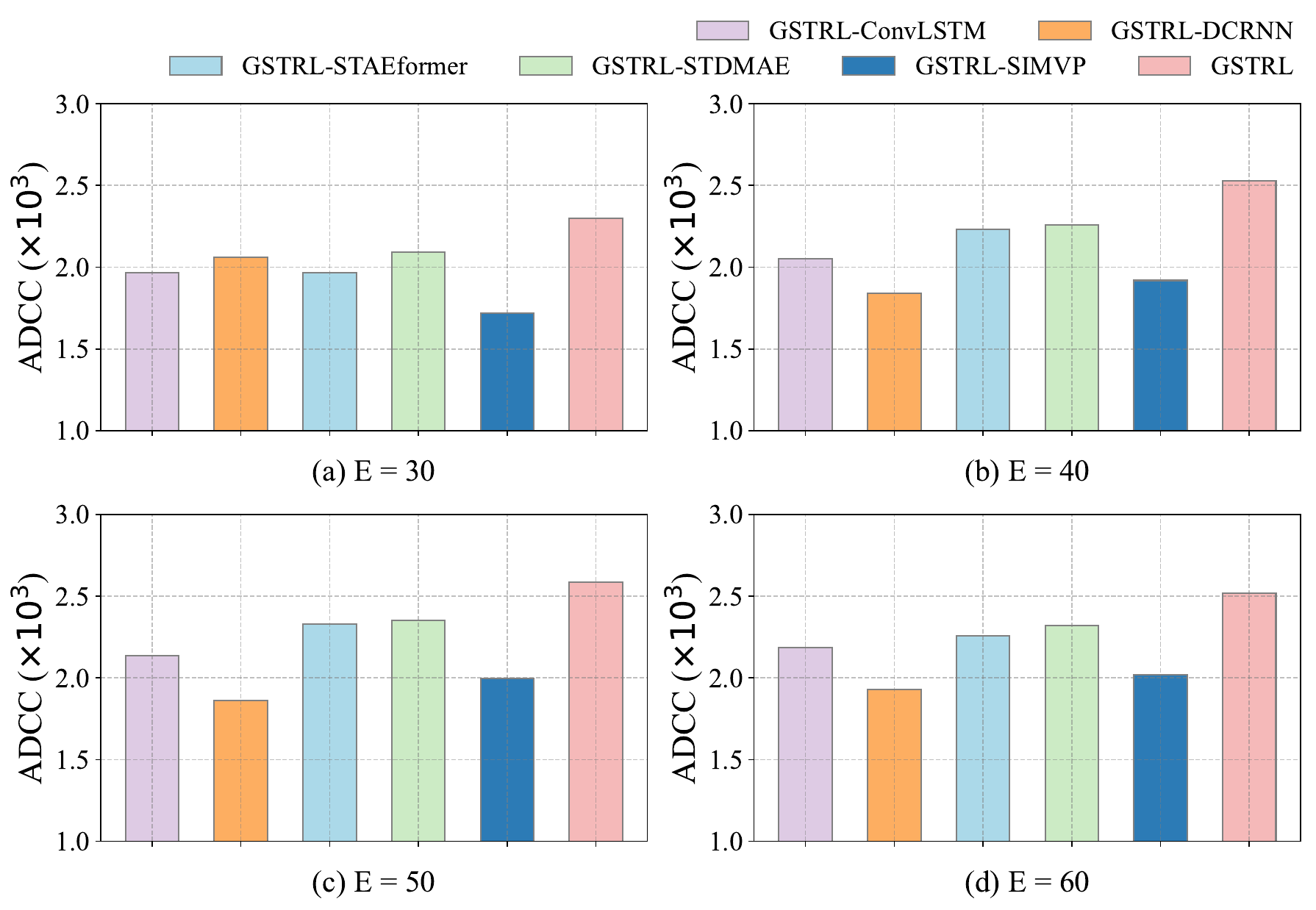}  
  \vspace{-1em}
  \caption{ADCC on different GSTRL variants.}  
  \label{fig:convergence}
  \vspace{-1em}
\end{figure}

As shown in \textbf{Fig.~\ref{fig:convergence}}, GSTRL outperforms all variants across different energies. We attribute this to two factors: a) Data is grid-based, and adapting graph-based methods requires constructing an adjacency matrix, which can distort the characteristics of the grid. b) Compared to SIMVP, which extracts spatio-temporal features holistically, our module aggregates features at each time step. Given the structurally simple, yet time-varying flow, this aggregation proves to be more effective, despite relying on classical components.

\subsection{Hyperparameters Study (RQ4)}

In this subsection, we evaluate the effectiveness of GSTRL using the Happy Valley dataset under different hyperparameters, which consist of model parameters and network parameters. The initial energy is set to 40. EADS, the state-of-the-art baseline for this dataset, is used for comparison.

\begin{figure}[!t]  
  \centering
  \includegraphics[width=0.95\linewidth]{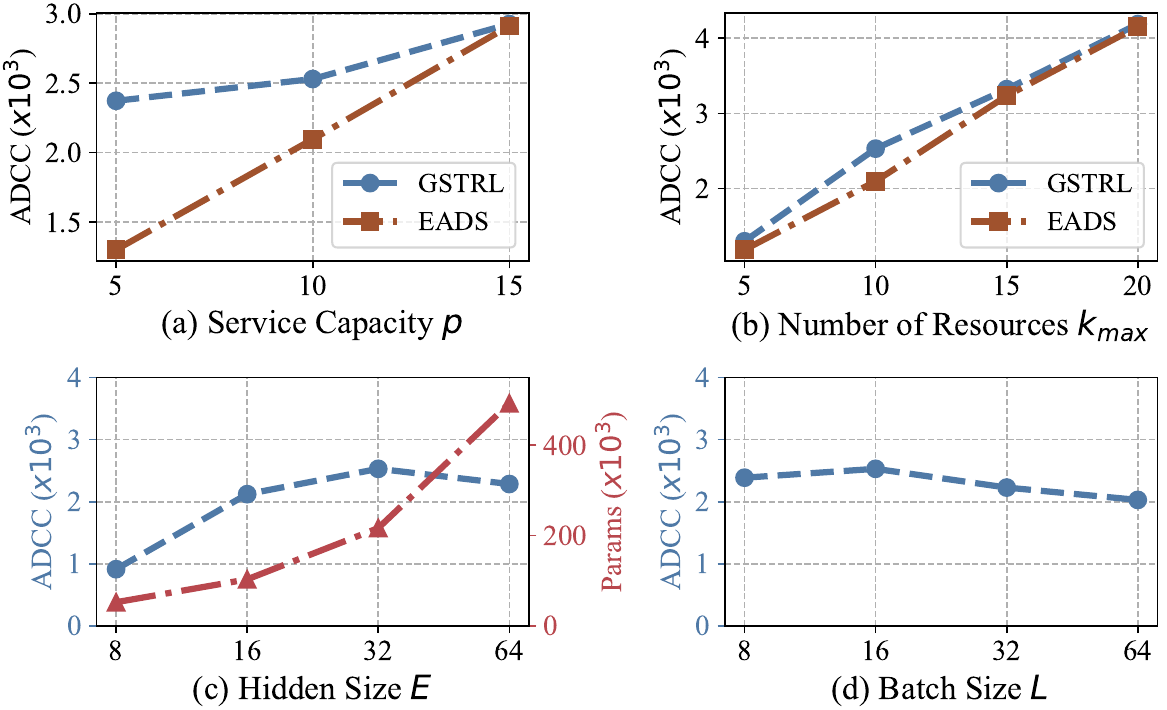}  
  \vspace{-1em}
  \caption{ADCC on different hidden sizes and batch sizes.}  
  \label{fig:hp_2}
\end{figure}

\vspace{-0.5em}
\subsubsection{Service Capacity.}

\textbf{Fig.~\ref{fig:hp_2} (a)} illustrates how ADCC varies with different service capacities. As the service capacity increases, the performance gap has narrowed. However, when the service capacity is low, GSTRL significantly outperforms EADS, highlighting its advantage in multi-agent collaboration.

\vspace{-0.5em}
\subsubsection{Number of Resources.}

\textbf{Fig.~\ref{fig:hp_2} (b)} presents the ADCC for different numbers of resources. As shown, GSTRL outperforms EADS across all configurations. Furthermore, increasing the number leads to greater ADCC, as more locations can be served.

\vspace{-0.5em}
\subsubsection{Hidden Size and Batch Size.}

We investigate the impact of hidden size and batch size with default settings. \textbf{Fig.~\ref{fig:hp_2} (c)} reveals that a small hidden size reduces performance, and a hidden size of 64 achieves comparable performance but incurs higher cost.

In \textbf{Fig.~\ref{fig:hp_2} (d)}, we examine the impact of batch size. The trend shows that increasing the batch size from 8 to 16 leads to improved performance, while increasing from 16 to 32 causes a slight decline, which suggests that a batch size of 16 may be optimal. The curve exhibits a smooth trend, indicating that GSTRL is insensitive to variations in batch size, highlighting its strong robustness.



\begin{figure}[!t]  
  \centering
  \includegraphics[width=0.95\linewidth]{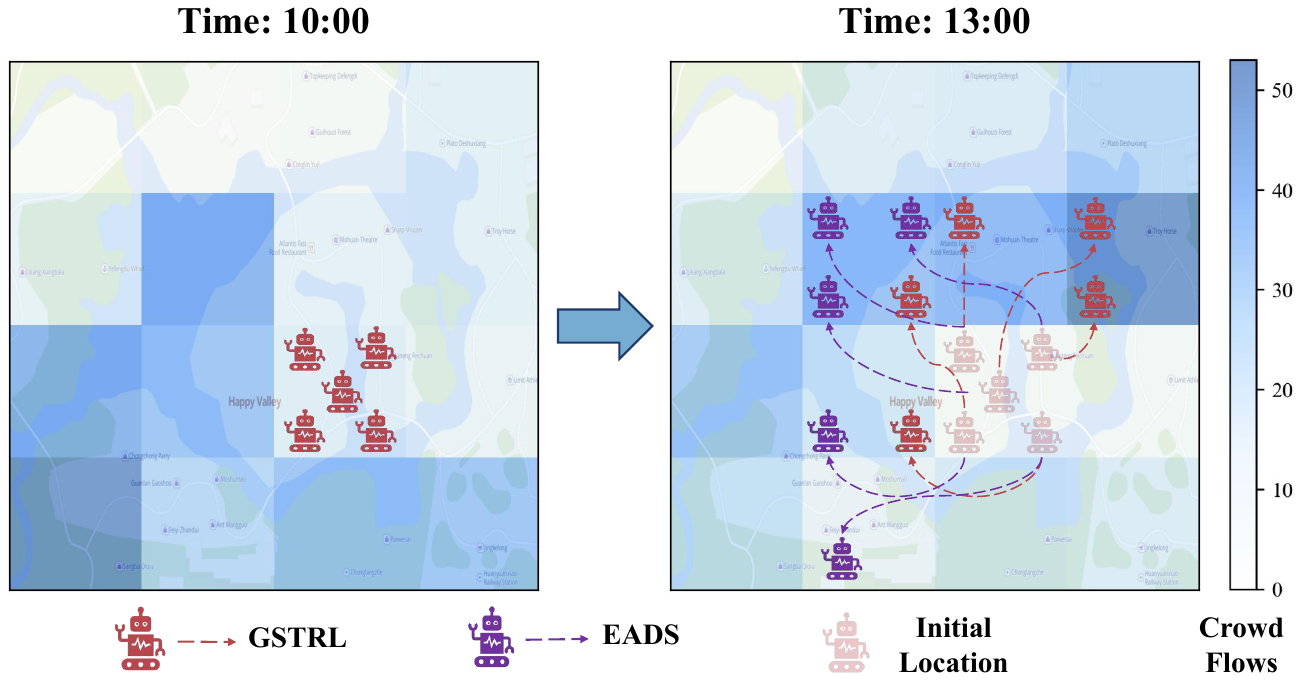}  
  \caption{Strategies visualization of GSTRL compared with EADS in Happy Valley dataset.}  
  \label{fig:case_study}
  \vspace{-1em}
\end{figure}

\subsection{Case Study (RQ5)}

We perform the allocation and visualize the process from 10:00 AM to 1:00 PM on September 13, 2018. To provide a more intuitive view, we merge the grid cells into a 4x4 grid with unequal grid dimensions and plot a heat map based on the maximum crowd flow within the merged cells. The study is conducted with 5 resources, a service capacity of 10, and an energy limit of 40. We compare GSTRL with the state-of-the-art EADS baseline on the Happy Valley dataset.

As shown in \textbf{Fig.~\ref{fig:case_study}}, the process begins at 10:00 AM with five resources from the resource depot. By 1:00 PM, resources allocated using GSTRL are strategically placed in areas of higher crowd density compared to those allocated with EADS. This superiority stems from GSTRL's proactive coordination mechanism: while EADS merely reacts to immediate localized demands and occasionally clusters redundantly in single hotspots, GSTRL leverages its spatio-temporal module to anticipate crowd movements. Governed by the potential game formulation, the resources collaboratively distribute themselves across multiple high-demand centers, maximizing coverage under their capacity without exhausting their energy limits.
In addition, we developed a CPRA system and visualized the deployment of GSTRL on the TaxiBJ dataset, as detailed in \textbf{Appendix}.

\vspace{-1em}
\section{Related Work}

\noindent \textbf{Dynamic Resource Allocation.}
Dynamic resource allocation has been studied across various urban applications, such as ambulance deployment~\cite{Ji2019pra}, electric bus scheduling~\cite{Yan2024ref}, and disaster response~\cite{mohammed2022disaster}. Early studies primarily relied on exact optimization or heuristic algorithms~\cite{bertsimas2019emergency, zhao2019adaptive}, which offer foundational solutions but excel in static environments. As urban dynamics intensify, these methods fail to adequately capture the rapid fluctuations of demands. Although recent approaches~\cite{Ruan2020pra} incorporate real-time data to improve scheduling efficiency, they typically optimize resource movements in an isolated manner and lack foresight regarding future rewards. Crucially, real-world public resources have strictly limited service capacities. When crowd flows exhibit extreme spatio-temporal heterogeneity, isolated dispatches inevitably lead to localized service failures or redundant deployments~\cite{Baek2022pra}. To address this gap, our work introduces the CPRA problem, shifting the paradigm from independent planning to collaborative allocation.

\noindent \textbf{Reinforcement Learning for Linear Programming.}
Resource allocation problems are traditionally formulated as Integer Linear Programming (ILP) or combinatorial optimization tasks. However, due to their NP-hard nature, classical LP solvers suffer from prohibitive computational costs in large-scale scenarios. Consequently, RL has emerged as a powerful alternative for approximating solutions to these programming problems by maximizing long-term discounted returns. Algorithms such as SAC~\cite{Zheng2024rlbas} and hierarchical frameworks like HEM~\cite{Wang2024rlbas} have been widely applied to dynamic environments. Despite their promise, directly applying standard RL to the ILP of collaborative allocation faces severe scalability bottlenecks. The service capacity constraints and the necessity for multi-agent coordination create a highly complex state-action space, causing existing RL methods to frequently converge to suboptimal solutions. To overcome these challenges, our GSTRL framework bridges mathematical programming and RL by employing potential game theory. By aligning the system's objective with individual agent rewards through a potential function, we guarantee theoretical convergence towards a near-optimal Nash equilibrium.

\noindent \textbf{Crowd Flow Prediction.}
Accurate crowd flow prediction is a fundamental prerequisite for proactive urban resource management, requiring sophisticated modeling of both complex spatial correlations and non-stationary temporal dependencies. While early methods struggled to capture cross-region influences, modern deep learning approaches~\cite{zhang2018gan, yao2018rev, Tan2024sim} have significantly improved prediction accuracy. However, most existing works treats flow prediction as an isolated, standalone objective, rather than seamlessly integrating it with downstream decision-making tasks like dynamic resource allocation. In the context of CPRA, merely forecasting values is insufficient, the multi-agent system must dynamically react to these evolving distributions under capacity limits. To bridge this gap, our framework focuses on deep spatio-temporal feature extraction. By fusing the extracted spatio-temporal representations of crowd flows with the status of resources, GSTRL effectively translates predictive insights into constraint-aware allocation strategies.

\section{Conclusion and Future Work}

In this paper, we introduce the CPRA problem and model it as a potential game with a potential function to align individual utilities with global outcomes. We propose a GSTRL framework to solve this problem, which incorporates initial state embedding, spatio-temporal feature extraction, and related modules to capture the dynamics of the system. Extensive experiments on two real-world datasets demonstrate the superiority of GSTRL over baselines across varying parameter settings and the positive impact of its key components. In future work, we aim to optimize energy consumption and explore more on-policy learning scenarios jointly.

\begin{acks}
This work is supported by the National Natural Science Foundation of China (No.~62402414, No.~62306033). This work is supported by the Guangdong Basic and Applied Basic Research Foundation (No.~2025A1515011994), Guangdong Provincial Project 2025D03J0014, Guangzhou Municipal Science and Technology Project (No.~2023A03-J0011), the Guangzhou Industrial Information and Intelligent Key Laboratory Project (No.~2024A03J0628), and Guangdong Provincial Key Lab of Integrated Communication, Sensing and Computation for Ubiquitous Internet of Things (No.~2023B1212010007). This work is also supported by Tianjin Natural Science Foundation Project (No. 25JCYBJC01540).
\end{acks}

\appendix

\begin{figure}[!t]  
  \centering
  \includegraphics[width=0.95\linewidth]{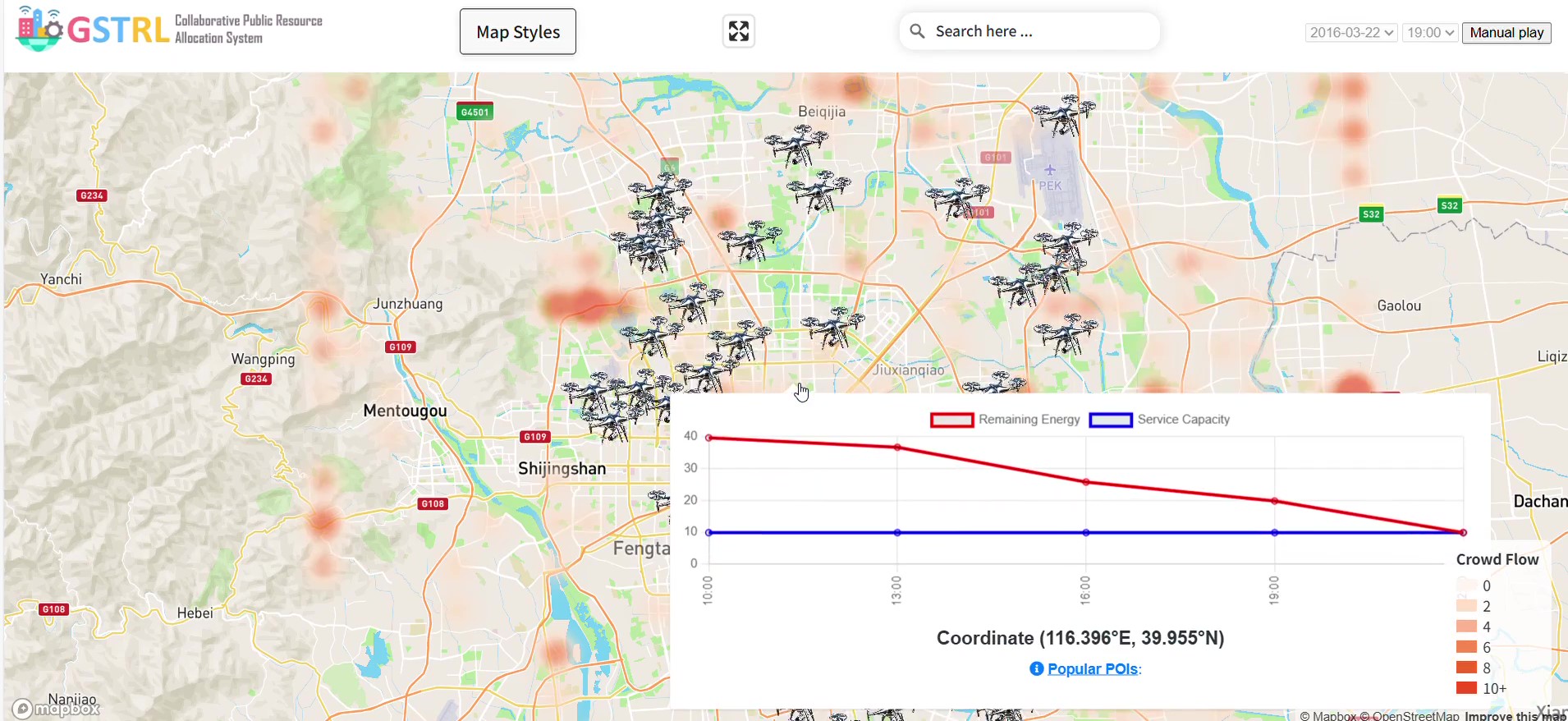}  
  \caption{Web platform of CPRA system.}  
  \label{fig:rea}
  \vspace{-1.5em}
\end{figure}

\section{Real‑World Deployment Visualization}
\label{app:realworld}

To provide an intuitive view of the behavior of GSTRL, we perform a deployment visualization on TaxiBJ. We first aligned the flow and their geographic coordinates with a real map of Beijing. Public resources are modeled as UAVs unconstrained by the road network, and the scenario is framed as drone delivery units that can respond directly. Throughout this study, we fix the initial energy to 40, deploy 50 resources, and set their service capacity to 10. 

Using the trained GSTRL, we replay the test period from March 22 to March 31, 2016 and record the location of each public resource at each time step. \textbf{Fig.~\ref{fig:rea}} overlays their locations on the heatmap. 
As the demand distribution changes over time, most UAVs migrate visibly to the densest regions, highlighting the superior performance of GSTRL. We also plot the remaining energy and service capacity traces for each public resource, confirming that the policy satisfies two constraints of CPRA, every UAV retains sufficient energy to return to the resource depot at the end of each day and the service capacity remains constant over time. 


\section{GenAI Disclosure}
We used a large language model (LLM) to assist with grammar editing of this manuscript. The LLM was not
used to generate scientific content, results, or conclusions. All technical content and claims were produced and verified by the authors.

\bibliographystyle{ACM-Reference-Format}
\bibliography{sample-base}

\end{document}